\documentclass[nohyperref]{article}

\usepackage{microtype}
\usepackage{graphicx}
\usepackage{subfigure}
\usepackage{booktabs} 
\usepackage{algorithm}
\usepackage{hyperref}

\usepackage[accepted]{icml2022}

\usepackage{amsmath}
\usepackage{amssymb}
\usepackage{mathtools}
\usepackage{amsthm}
\usepackage{bbding}

\usepackage[capitalize,noabbrev]{cleveref}

\theoremstyle{plain}
\newtheorem{theorem}{Theorem}[section]

\newtheorem{lemma}[theorem]{Lemma}

\theoremstyle{definition}

\theoremstyle{remark}

\usepackage[textsize=tiny]{todonotes}

\icmltitlerunning{Federated Learning with Positive and Unlabeled Data}

\begin{document}

\twocolumn[
\icmltitle{Federated Learning with Positive and Unlabeled Data}



\icmlsetsymbol{equal}{*}

\begin{icmlauthorlist}
\icmlauthor{Xinyang Lin}{equal,xju}
\icmlauthor{Hanting Chen}{equal,huawei}
\icmlauthor{Yixing Xu}{huawei}
\icmlauthor{Chao Xu}{pku}
\icmlauthor{Xiaolin Gui}{xju}
\icmlauthor{Yiping Deng}{huawei2}
\icmlauthor{Yunhe Wang$^{\dag}$}{huawei}
\end{icmlauthorlist}

\icmlaffiliation{huawei}{Huawei Noah’s Ark Lab}
\icmlaffiliation{huawei2}{Central Software Institution, Huawei Technologies}
\icmlaffiliation{xju}{Faculty of Electronic and Information Engineering, Xi’an Jiaotong University}
\icmlaffiliation{pku}{Key Lab of Machine Perception (MOE), Department of Machine Intelligence, Peking University, China}

\icmlcorrespondingauthor{Yunhe Wang$^{\dag}$}{yunhe.wang@huawei.com}

\icmlkeywords{Machine Learning, ICML}

\vskip 0.3in
]



\printAffiliationsAndNotice{\icmlEqualContribution} 

\begin{abstract}
	We study the problem of learning from positive and unlabeled (PU) data in the federated setting, where each client only labels a little part of their dataset due to the limitation of resources and time. Different from the settings in traditional PU learning where the negative class consists of a single class, the negative samples which cannot be identified by a client in the federated setting may come from multiple classes which are unknown to the client. Therefore, existing PU learning methods can be hardly applied in this situation. To address this problem, we propose a novel framework, namely Federated learning with Positive and Unlabeled data (FedPU), to minimize the expected risk of multiple negative classes by leveraging the labeled data in other clients. We theoretically analyze the generalization bound of the proposed FedPU. Empirical experiments show that the FedPU can achieve much better performance than conventional supervised and semi-supervised federated learning methods. Code is available at \url{https://github.com/littleSunlxy/FedPU-torch}
\end{abstract}

\section{Introduction}

With the development of edge devices (\emph{e.g.}, cameras, microphones, and GPS), more and more decentralized data are collected and locally stored by different users. Due to the privacy and transmission concerns, users are unwilling or not allowed to share the data with each other. In this case, classical machine learning scheme can hardly learn a globally effective model for all the users. Therefore, federated learning~\cite{Mcmahan2017communication} is proposed to derive a model with high performance in the central server by leveraging multiple local models trained by users (clients) themselves, which ensures the privacy of the local data.  

Typically, there is a common assumption in federated learning that the local data (private data) stored on user devices is well refined (\emph{i.e.}, all of the local data is labeled with ground truth). However, considering the limitation of time and resources, only part of the private data in each client are labeled in reality. To this end, some of the previous works were proposed to address this federated learning problem following a semi-supervised scheme. ~\cite{jeong2020federated} proposed the FedMatch algorithm which introduced a new inter-client consistency loss and decomposed the parameters for labeled and unlabeled data. ~\cite{zhang2020benchmarking} managed to solve this problem by conducting a novel grouping-based model average method and improved the convergence efficiency. ~\cite{itahara2020distillation} proposed a distillation-based algorithm to exchange the local models among each client and learned the unlabeled data by pseudo labels. Although these methods can successfully address the semi-supervised learning problem for federated learning, they assume that each class has labeled samples in each client. However, in real world applications, users from each client may only label part of categories due to their limited ability.

To address the aforementioned problem, we consider a more general setting of federated learning with unlabeled data: 1) each client only labels \emph{part of} their own data which comes from \emph{part of} the classes; 2) there are no data in the central server; 3) nothing except parameters of models can be exchanged between clients and the central server. Note that the first constraint of our setting meets the problem of learning from positive and unlabeled (PU) data. Existing PU methods~\cite{liu2003building,liu2015classification,xu2017multi} focused on solving the PU problem which regard the negative class (class that contains no labeled samples) as a single class. However, since negative class in one client may consist of multiple positive classes in other clients, there are multiple negative classes in one client in federated learning, which results in a multiple-positive-multiple-negative PU (MPMN-PU) learning problem and cannot be solved using existing PU learning framework.  

\begin{figure*}[t]
	\centering
	\includegraphics[width=0.9\linewidth]{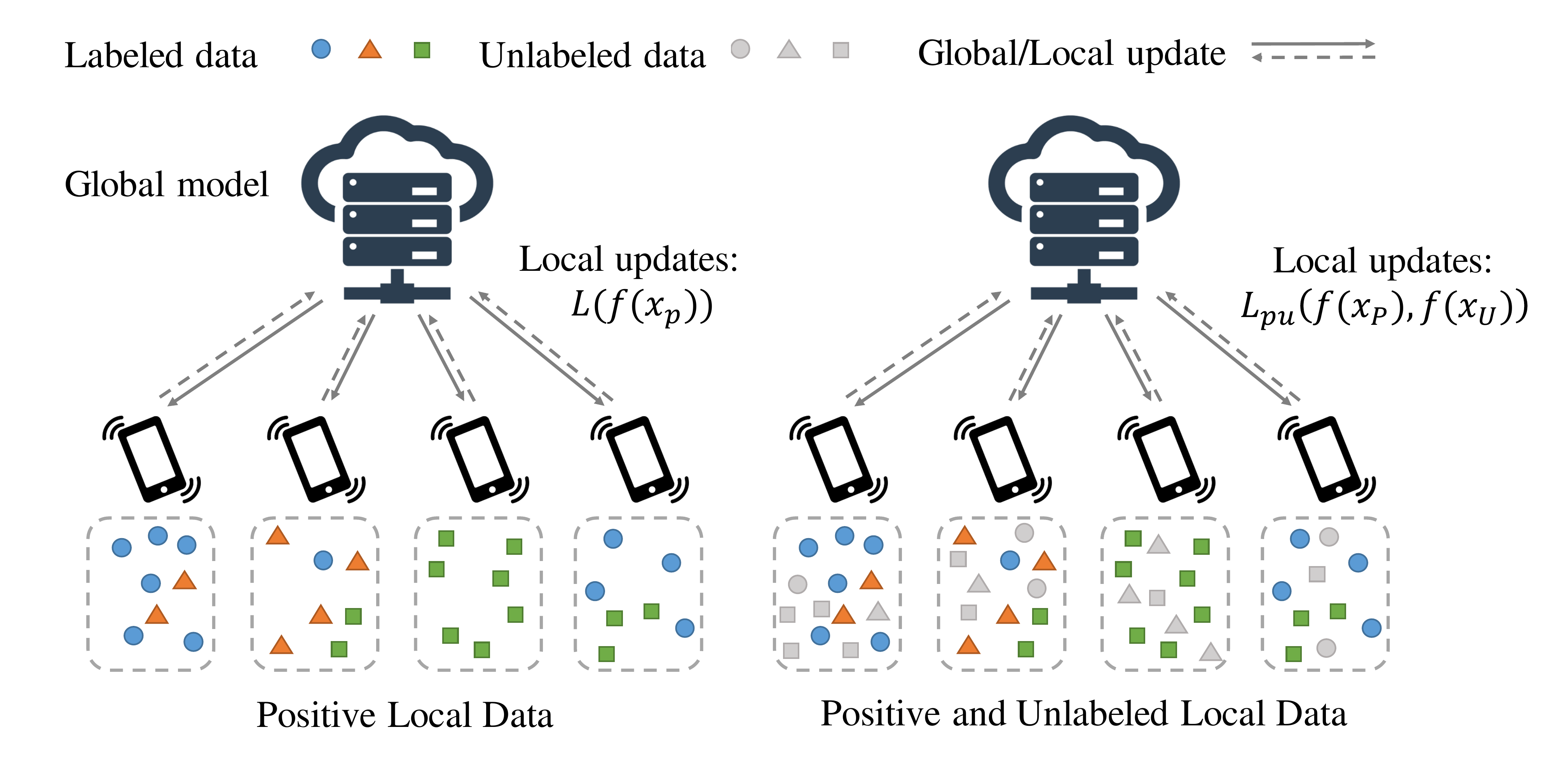}
	\caption{Illustration of the conventional federated learning (left) and the proposed method (right). Conventional federated learning method only learns from labeled data. In contrast, we propose the federated learning with positive and unlabeled data to fully inherit the information from the unlabeled data.}
	\label{fig}
\end{figure*}

In this paper, we propose the Federated learning with Positive and Unlabeled data (FedPU) algorithm, where the local model in each client is trained with MPMN-PU data. We first analyze the expected risk of each class in each client and show that the risks of multiple negative classes can be successfully minimized by leveraging unlabeled data in this client and labeled data in other clients, which is shown in Figure~\ref{fig}. Moreover, we present a generalization bound of proposed FedPU and show that the FedPU algorithm is no worse than $C\sqrt{C}$ times (where $C$ denotes the number of classes) of the fully-supervised model in federated setting. Experiments on MNIST and CIFAR datasets empirically show that the proposed method can achieve better performance than existing federated learning algorithms.

\section{Related Works}
In this section, we briefly review the related works about the federated learning and positive-unlabeled learning. 

\subsection{Federated Learning}

Federated learning is firstly proposed by~\cite{Mcmahan2017communication} in order to collaboratively learn a model without collecting data from the participants.~\cite{bonawitz2017practical} proposes the secure aggregation based on the concept of the Secure Multiparty Computation (SMC) algorithm, which aggregates private values of mutually distrustful parties without revealing information about their private values.~\cite{geyer2017differentially} introduces client-level differential privacy to prevent any client from trying to reconstruct the private data of another client by exploiting the global model in federated learning.~\cite{yang2019federated} considers the statistical challenge of the heterogeneity of data from users in practical settings that cooperation are conducted on low-quality, incomplete and insufficient data.~\cite{Mcmahan2017communication} proposes the Federated Averaging (FedAvg) algorithm, which performs aggregating algorithm by averaging model updates from participants.~\cite{ghosh2020efficient} proposes the Iterative Federated Clustering Algorithm (IFCA), which optimizes the weights for each client by estimating the cluster identities. In the statistical heterogeneity context,~\cite{acar2020federated} targets the non i.i.d client dataset problem in federated learning and aligns the loss surfaces of clients with a novel dynamic regularizer.~\cite{acar2021debiasing} analyzes different personalization methods and uses gradient correction algorithms to ensure convergence by being agnostic to heterogeneity levels. Recently, several researches~\cite{li2018federated,karimireddy2020scaffold,sattler2019robust} focus on improving model performance on non-iid data. 

\subsection{Positive and Unlabeled Learning}

Various effective algorithms have been developed to solve the PU learning problem.~\cite{liu2003building} proposes the two-step technique based on the assumption that all the positive samples are similar to the labeled examples and the negative samples are very different from them.~\cite{liu2015classification} introduces an biased PU learning methods, which treats the unlabeled samples as negative ones with label noise.~\cite{lee2003learning} regards the unlabeled data as negative data with smaller weights, then performed logistic regression after weighting the samples to handle the situation that noise rate is greater than a half. In order to avoid tuning the weights, ~\cite{elkan2008learning} regards unlabeled data as weighted positive and negative data simultaneously.~\cite{du2014analysis} proposes the unbiased risk estimator and~\cite{kiryo2017positive} makes a progress by proposing a non-negative risk estimator for PU learning to mitigate the overfitting problem when using a flexible model.~\cite{garg2021mixture} investigated methods for mixture proportion estimation and PU classification.~\cite{xu2017multi} adapted PU learning to the setting with multi-class classification problem. These methods regard the negative class as a single class, which is reasonable when there is only a single dataset. However, in federated learning, the datasets are distributed in different clients, where samples from the negative classes in one client may become positive in another client since different clients are free to label their data. To this end, an effective PU learning algorithm for the federated setting is urgently required.   

\section{Method}

In this section, we study federated learning problem under the MPMN-PU learning setting for each client.   

\subsection{Problem Setup}

Here we first introduce the notations in federated learning, where there are $K$ different clients and one central server. Given the data space $\mathcal{S}$ and the hypothesis space of parameters $\mathcal{W}$, the training data is distributed on $K$ different clients and is generated from the data space $\mathcal{S}$, which is denoted as $\{\mathbf{S}_k\}^{K}_{k=1}\in\mathcal{S}$. Denote $T$ as the number of communication rounds and $w_t\in \mathcal{W}$ as the weight matrix in the central server in time $t\in\{1,...,T\}$, the weights $w_t$ is first transferred from the central server to each client, and then updated using the training data in each client respectively and derive $K$ different weights:   
\begin{equation}
w^k_{t+1} \leftarrow\mbox{ClientUpdate}(k,w_t),
\end{equation}
where $w_{t+1}^k, k\in\{1,...,K\}$ is the updated weights from client $k$ and the client update stage is a conventional training method for updating the gradient. After that, the updated weights are then transferred back to renew the weight matrix in central server:   
\begin{equation}
w_{t+1} \leftarrow \sum_{k=1}^K \frac{n^k}{n} w^{k}_{t+1},
\end{equation}
where $n^k$ is the number of training samples in client $k$ and $n=\sum_{k=1}^{K} n^k$ is the number of all the training samples.

In the traditional federated learning setting, the training data in each client is fully labeled. Nevertheless, samples are not always fully labeled in many real world scenarios because of the time and resources limitation in each client. Specifically, the training data $\mathbf{S}_k$ in client $k$ consists of positive data $\mathbf{P}_k$ and unlabeled data $\mathbf{U}_k$, which can be formulated as:
\begin{equation}
\mathbf{S}_k = \mathbf{P}_k \cup \mathbf{U}_k, \quad k = 1,\dots,K.
\end{equation}
Given the set of classes as $\mathbf{C}=\{1,...,C\}$ in which $C$ is the total number of classes, the set of classes of positive data (i.e. the positive classes) in client $k$ is denoted as $\mathbf{C}_{\mathbf{P}_k}$, while the negative classes is denoted as $\mathbf{C}_{\mathbf{N}_k}$, where $\mathbf{C}_{\mathbf{P}_k} \bigcup \mathbf{C}_{\mathbf{N}_k}= \mathbf{C}$. In other words, each client can only identify part of the classes from the dataset $\mathbf{S}_k$. Besides, only a portion of the data in the positive classes can be labeled since the data is too much to be fully labeled. Therefore, there exists unlabeled data from not only the negative classes but also the positive classes, \emph{i.e.}, $\mathbf{C}_{\mathbf{U}_k} = \mathbf{C} = \mathbf{C}_{\mathbf{P}_k} \bigcup \mathbf{C}_{\mathbf{N}_k}$. Specifically, we have:  
\begin{equation}
\begin{aligned}
&\forall x \in \mathbf{P}_k, \textbf{Class}(x) \in \mathbf{C}_{\mathbf{P}_k};\\ &\forall x \in \mathbf{U}_k, \textbf{Class}(x) \in \mathbf{C}_{\mathbf{P}_k} \bigcup \mathbf{C}_{\mathbf{N}_k}.
\end{aligned}
\end{equation}
Note that different clients have different set of positive classes, and all of the positive classes should cover the whole classes in the dataset, \emph{i.e.}, $ \bigcup_{\mathbf{P}_k} \mathbf{C}_{\mathbf{P}_k} = \mathbf{C}$.

In this setting, the conventional federated learning algorithms cannot be directly applied. Fortunately, PU (Positive and Unlabeled) learning~\cite{liu2003building} has been proposed to solve this problem. 
However, traditional PU learning methods regard the negative class as a single class, which is inappropriate in federate learning since negative class in one client may consists of multiple positive classes in other clients. Therefore, we meet a MPMN (Multi-Positive and Multi-Negative) PU learning problem, which cannot be directly handled by existing methods. 
\begin{algorithm}[t]
	\caption{The proposed FedPU learning algorithm.}\label{alg:FedAvg}
	\begin{algorithmic}[1]
		\REQUIRE Training dataset $\mathbf{S}_k$ in each client $k$ with $n^k$ training samples, class prior $\pi_i$ for each class $i=1,...,C$, communication round $T$ and training iteration $I$ for each client. 
		
		\STATE\textbf{Server executes:}
		\STATE Initialize the network $f(\mathbf{x};w_0)$.
		\FOR{each round $t = 1, 2, \dots, T$}
		\FOR{each client $k \in \{1,2,...,K\}$ \textbf{in parallel}}
		\STATE $w_{t+1}^k \leftarrow \text{ClientUpdate}(k, w_t)$ 
		\ENDFOR
		\STATE $w_{t+1} \leftarrow \sum_{k=1}^K \frac{n^k}{n} w_{t+1}^k$
		\ENDFOR
		
		\STATE \textbf{ClientUpdate($k, w_t$):}\ \ \  // \emph{Run on client $k$}
		\FOR{each local epoch $i$ from $1$ to $I$}
		\STATE Randomly select a batch of positive and unlabeled data $\{\mbox{x}^k\}$ from the dataset $\mathbf{S}_k$;
		\STATE Calculate the first term and second term in Eq.~\ref{equ:fedpurisk} using labeled data by $f(\mathbf{x}_P;w_t)$.
		\STATE Calculate the third term in Eq.~\ref{equ:fedpurisk} using unlabeled data by $f(\mathbf{x}_U;w_t)$.
		\STATE Minimize the loss function in Eq.~\ref{equ:fedpurisk} and update the weights $w^k_{t}$ according to the gradient.
		\ENDFOR
		\STATE Return the updated weight $w^k_{t+1}$ to server.
		\ENSURE The model $f(\mathbf{x};w_T)$ trained by PU data. 
	\end{algorithmic}
\end{algorithm}
\subsection{Federated Learning with Positive and Unlabeled Data}
To address the MPMN PU learning problem, we propose our FedPU (Federated learning with Positive and Unlabeled data) method. We assume to utilize FedAvg as the federated aggregation method for simplicity.

Here we first present our MPMN PU learning scheme in a single client (or without federate setting) for convenience. Denote the training samples as $\{(\mathbf{x}_i,\mathbf{y}_i)\}_{i=1}^{n}\in \mathbf{S}$. In classical multi-class classification, given the class prior $\pi_i = p(y=i), i=1,2,...C$, the classifier $f(\mathbf{x};w)$  (short as $f(\mathbf{x})$), in which $w$ is the parameter of the classifier, can be learned by minimizing the expected misclassification rate $R(f)$:   
\begin{equation}
\label{equ:risk}
R(f) = \sum_{i=1}^C \pi_i R_i(f) = \sum_{i=1}^C \pi_i  P_i(f(\mathbf{x}) \neq i),
\end{equation}
where $\sum_{i=1}^C \pi_i =1$ and $P_i(\cdot)$ denotes the probability calculated in $i$-th class samples. Therefore, $P_i(f(\mathbf{x}) \neq i)$ denotes the expected misclassification rate on $i$-th class.

However, in MPMN PU setting, only samples in a few classes are labeled in the training set for each client. Some of classes in Eq.~\ref{equ:risk} is unlabeled and the expected risk cannot be directly calculated in each client. Therefore, it is necessary to analyze the expected risk in the negative classes using the unlabeled data. Here we first introduce $R_{U}(f)$ to denote the sum of probability that the unlabeled samples does not belong to each of the negative class:   
\begin{equation}
\scriptsize
\label{equ:unlabel}
\begin{aligned}
R_{U}(f)=& \sum_{m\in \mathbf{C}_{\mathbf{N}}} P_U(f(\mathbf{x})\neq m)\\
=&\sum_{i\in \mathbf{C}_\mathbf{P}} \sum_{m\in \mathbf{C}_{\mathbf{N}}} \pi_i P_i(f(\mathbf{x})\neq m) + \sum_{j \in \mathbf{C}_{\mathbf{N}}} \sum_{m\in \mathbf{C}_{\mathbf{N}}}\pi_j P_j(f(\mathbf{x})\neq m)\\
= &\sum_{i\in \mathbf{C}_\mathbf{P}} \sum_{m\in \mathbf{C}_{\mathbf{N}}} \pi_i P_i(f(\mathbf{x})\neq m)+ \sum_{j  \in \mathbf{C}_{\mathbf{N}}}\pi_j P_j(f(\mathbf{x})\neq j)\\
&+ \sum_{ {j,m}  \in \mathbf{C}_{\mathbf{N}}, j\neq m}\pi_j P_j(f(\mathbf{x})\neq m),
\end{aligned}
\end{equation}
where $P_U(\cdot)$ denotes the probability calculated in unlabeled samples. Since the unlabeled samples may from both positive and negative classes, the probability $P_U(\cdot)$ can be separated into $\sum_{i\in \mathbf{C}_\mathbf{P}}P_i(\cdot)$ and $\sum_{j\in \mathbf{C}_\mathbf{N}}P_j(\cdot)$. Finally, $R_{U}(f)$ can be divided into three terms, where the first term is the probability of positive data have not been classified to the set of negative classes, the second term is the probability of negative data have not been classified to the corresponding negative class, and the third term is the probability of negative data have not been classified to the other negative classes. Note that the second term is exactly the expected risk in the negative classes, Eq.~\ref{equ:risk} can be reformulated as:
\begin{equation}
\scriptsize
\label{equ:purisk}
\begin{aligned}
R(f)=& \sum_{i\in \mathbf{C}_\mathbf{P}} \pi_i R_i(f) + \sum_{j\in \mathbf{C}_{\mathbf{N}}} \pi_j R_j(f) \\
= &  \sum_{i\in \mathbf{C}_\mathbf{P}} \pi_i R_i(f) + R_{U}(f) -  \sum_{i\in \mathbf{C}_\mathbf{P}} \sum_{m\in \mathbf{C}_{\mathbf{N}}} \pi_i P_i(f(\mathbf{x})\neq m) \\
&-\sum_{ {j,m} \in \mathbf{C}_{\mathbf{N}}, j\neq m}\pi_j P_j(f(\mathbf{x})\neq m)\\
= &  \sum_{i\in \mathbf{C}_\mathbf{P}} \pi_i [P_i(f(\mathbf{x})\neq i) - \sum_{m\in \mathbf{C}_{\mathbf{N}}}P_i(f(\mathbf{x})\neq m)]  \\
&+\sum_{m\in \mathbf{C}_{\mathbf{N}}} P_U(f(\mathbf{x})\neq m) -  \sum_{ {j,m} \in \mathbf{C}_{\mathbf{N}}, j\neq m}\pi_j P_j(f(\mathbf{x})\neq m).
\end{aligned}
\end{equation}

Through calculating the $R_{U}(f)$ in unlabeled data, we can successfully obtain the expected risk in the negative classes. Now we are ready to solve the federated learning problem with MPMN-PU data. Here we turn to the federated learning setting, the expected risk can be formulated as:
\begin{equation}
\label{equ:fedrisk}
R^{all}(f) = \sum_{k=1}^K R^k(f),
\end{equation}
where $R^k(f)$ denote the expected risk in client $k$. Given Eq.~\ref{equ:purisk}, the corresponding expectation of the expected risk using in PU setting can be reformulated as:
\begin{equation}
\scriptsize
\label{equ:fedrisk1}
\begin{aligned}
\mathbb{E}[R^k(f)] =  &\sum_{i\in \mathbf{C}_{\mathbf{P}_k}} \pi_i \mathbb{E}_i^k\left[P(f(\mathbf{x})\neq i) - \sum_{m\not\in \mathbf{C}_{\mathbf{P}_k}}P(f(\mathbf{x})\neq m)\right]  \\
&+\sum_{m\not\in \mathbf{C}_{\mathbf{P}_k}}\mathbb{E}_U^k\left[ P(f(\mathbf{x})\neq m)\right]
\\&-  \sum_{ {j,m} \not\in \mathbf{C}_{\mathbf{P}_k}, j\neq m} \pi_j \mathbb{E}_j^k\left[ P(f(\mathbf{x})\neq m)\right],
\end{aligned}
\end{equation}
where $\mathbb{E}_i^k$ means the expectation for the labeled data of $i$th class in client $k$, and $\mathbb{E}_U^k$ means the expected risk for unlabeled data in client $k$. 

Note that the federated MPMN-PU learning problem has several negative classes, which is fundamentally different with conventional PU learning problem~\cite{liu2003building, xu2017multi} whose negative class is a single class. We have an additional term $\sum_{ {j,m} \not\in \mathbf{C}_{\mathbf{P}_k}, j\neq m} \pi_j \mathbb{E}_j\left[ P(f(\mathbf{x})\neq m)\right]$ in Eq.~\ref{equ:fedrisk1}. Actually, this term denotes the misclassifiation loss between the negative classes, which have not appeared in traditional PU problem since they only have a single negative class.

Considering that the negative classes are unlabeled, it is difficult to directly calculate $\sum_{ {j,m} \not\in \mathbf{C}_{\mathbf{P}_k}, j\neq m} \pi_j \mathbb{E}_j\left[ P(f(\mathbf{x})\neq m)\right]$. Fortunately, we have  $\bigcup_{\mathbf{P}_k} \mathbf{C}_{\mathbf{P}_k} = \mathbf{C}$, which means that although we have no information for the negative class in one client, there exists labeled data for these classes in other clients. Since the weights in central server is derived from the combination of each client, we can calculate this term by the labeled data in other clients. Specifically, assuming that data in the same class in different clients follows the same distribution, when updating the weights in client $k_1$, we abundant the term $\sum_{ {j,m} \mathbf{C}_{\mathbf{N}_{k_1}}, j\neq m} \pi_j \mathbb{E}^{k_1}_j\left[ P(f(\mathbf{x})\neq m)\right]$, while when updating the weights in client $k_2$, we add the term $\sum_{ {j,m} \in \mathbf{C}_{\mathbf{P}_{k_2}}, j\neq m} \pi_j \mathbb{E}^{k_1}_j\left[ P(f(\mathbf{x})\neq m)\right]$, where $\quad j\in \mathbf{C}_{\mathbf{N}_{k_1}}, j\in \mathbf{C}_{\mathbf{P}_{k_2}}$. According to the Eq.~\ref{equ:fedrisk}, since the weights in central server is derived from the
combination of each client, the overall risk $R(f)$ remains the same after applying this approximation.

By applying the above technique to Eq.~\ref{equ:fedrisk1}, we can successfully formulated the PU learning risk as:
\begin{equation}
\scriptsize
\label{equ:fedpurisk}
\begin{aligned}
\mathbb{E}[R^k(f)] =&  \sum_{i\in \mathbf{C}_{\mathbf{P}_k}} \pi_i \mathbb{E}_i^k\left[P(f(\mathbf{x})\neq i) - \sum_{m\not\in \mathbf{C}_{\mathbf{P}_k}}P(f(\mathbf{x})\neq m)\right] \\
&+\sum_{m\not\in \mathbf{C}_{\mathbf{P}_k}}\mathbb{E}_U^k\left[ P(f(\mathbf{x})\neq m)\right]\\&-  \sum_{k_q\neq k}\quad\sum_{ i\in \mathbf{C}_{\mathbf{P}_{k}},i,m\not\in \mathbf{C}_{\mathbf{P}_{k_q}},i\neq m} \pi_i\mathbb{E}_i^k   \left[ P(f(\mathbf{x})\neq m)\right] ,
\end{aligned}
\end{equation}
where the first and second terms are the risks from the current client while the second term is derived from other clients. Different with Eq.~\ref{equ:fedrisk1} that contains risk of negative classes, the above equation can be easily minimized since it only consists of the risk of positive data and unlabeled data. 
Therefore, the overall expected risk in Eq.~\ref{equ:risk} can be minimized by minimizing the above risk in each client. Algorithm~\ref{alg:FedAvg} shows the detailed procedure of the proposed FedPU method. 

\subsection{Theoretical Analysis}

In this section, we analyze the generation bound of the proposed FedPU. We first evaluate the bound in each client. Then the overall bound can be derived by summing these bounds. Note that the proof of theorems and lemma can be found in the supplementary materials.

Since the Eq.~\ref{equ:fedpurisk} has three terms, we begin with the first and second terms. 

\begin{theorem}
	\label{theorem1}
	Fix $f\in\mathcal{F}$, for any $0<\delta<1$, with probability at least $1-\delta$, the generalization bound holds:
	\begin{equation}
	\small
	\begin{aligned}
	&\mathbb{E}_i^k\left[P(f(\mathbf{x})\neq i) - \sum_{m\not\in \mathbf{C}_{\mathbf{P}_k}}P(f(\mathbf{x})\neq m)\right]\\& - \frac{1}{n_i^k}\sum_{j=1}^{n_i^k}\left[ P(f(\mathbf{x}_j)\neq i) - \sum_{m\not\in \mathbf{C}_{\mathbf{P}_k}}P(f(\mathbf{x}_j)\neq m) \right] \\
	\leq& 2CV(\sum_{s\in \mathbf{C}_{\mathbf{P}_k}} \frac{1}{\sqrt{n_s^k}} + \frac{1}{\sqrt{n_U^k}}) + \sqrt{\frac{log\frac{1}{\delta}}{2n_i^k}},
	\end{aligned}
	\end{equation}
	where $i\in \mathbf{C}_{\mathbf{P}_k}$, $V$ is a constant related to the VC-dimension of $f$ and the bound of function $f$, $n_i^k$ and $ n_U^k$ denotes the number of samples in $i$-class and unlabeled classes in $k$-th client, respectively.  
\end{theorem}

\begin{theorem}
	\label{theorem2}
	Fix $f\in\mathcal{F}$, for any $0<\delta<1$, with probability at least $1-\delta$, the generalization bound holds:
	\begin{equation}
	\begin{aligned}
	\mathbb{E}_i^k   \left[ P(f(\mathbf{x})\neq m)\right] - \frac{1}{n_i^k}\sum_{j=1}^{n_i^k}P(f(\mathbf{x}_j)\neq m) \\\leq CV(\sum_{s\in \mathbf{C}_{\mathbf{P}_k}} \frac{1}{\sqrt{n_s^k}} + \frac{1}{\sqrt{n_U^k}}) + \sqrt{\frac{log\frac{1}{\delta}}{2n_i^k}}.
	\end{aligned}
	\end{equation}
\end{theorem}

\begin{table*}[t]
	\begin{center}
		\caption{Classification result on iid data.}
		\label{table:iid}
		\begin{tabular}{|c|c|c|c|c|c|c|}
			\hline
			\textbf{Num of Clients} &\textbf{Num of P-class}&\textbf{Overlap} &\textbf{Baseline-1}&\textbf{Proposed Method}&\textbf{Baseline-2}  \\
			\hline
			\hline
			10 &2&\Checkmark & 85.47\% & 92.50\% & 97.95\% \\
			\hline
			4 & 6&\Checkmark & 92.10\% & 95.08\% & 98.05\% \\
			\hline
			2 & 9&\Checkmark &93.15\% & 95.37\% & 98.20\% \\
			\hline
			10 & 1&\XSolidBrush&37.13\% & 84.15\% & 97.95\% \\
			\hline
			
			5 & 2&\XSolidBrush&73.41\% & 93.45\% & 98.03\% \\
			\hline
			
			2 & 5&\XSolidBrush&74.00\% & 93.73\% & 98.20\% \\
			\hline
		\end{tabular}
	\end{center}

\end{table*}

Theorem~\ref{theorem1} and~\ref{theorem2} presents the classical generalization bound for the labeled data in each class, which can be summarized to get the error bound for the first two terms in Eq.~\ref{equ:fedpurisk}. These bounded is related to the number of training samples and the VC dimension of the function $f$.

However, it is difficult to derive the error bound of the last term in Eq.~\ref{equ:fedpurisk} since the expectation is calculated on unlabeled data, so we decompose this term using the following lemma.

\begin{lemma}
	\label{lemma3}
	Define 
	\begin{equation}
	P'(f(\mathbf{x})\neq m) = \frac{k^{C_{U_k}}}{k^{C_{U_k}}+\prod_{i\not\in \mathbf{C}_{\mathbf{P}_k}}\vert k-i\vert} P(f(\mathbf{x})\neq m),
	\end{equation}
	where $C_{U_k}$ denotes the number of unlabeled classes in client $k$. The last term in Eq.~\ref{equ:fedpurisk} can be decomposed as:
	\begin{equation}
	\small
	\label{aaa}
	\begin{aligned}
	&\sum_{m\not\in \mathbf{C}_{\mathbf{P}_k}} \mathbb{E}_U^k\left[ P(f(\mathbf{x})\neq m)\right] 
	\\=&\sum_{i\in\mathbf{C}_{\mathbf{P}_k}} \pi_i (\frac{\prod_{i\not\in \mathbf{C}_{\mathbf{P}_k}}\vert k-i\vert}{k^{C_{U_k}}}) \sum_{m\not\in \mathbf{C}_{\mathbf{P}_k}} \mathbb{E}_i^k\left[ P'(f(\mathbf{x})\neq m)\right]\\
	&+ \sum_{m\not\in \mathbf{C}_{\mathbf{P}_k}} \mathbb{E}_U^k\left[ P'(f(\mathbf{x})\neq m)\right].
	\end{aligned}
	\end{equation}
\end{lemma}

Here we briefly explain the decomposition of the above lemma. We perform a transformation in the risk of unlabeled data, which introducing the term of labeled data in Eq.~\ref{aaa}. Therefore, based on Lemma~\ref{lemma3}, we can present the generalization bound for the unlabeled data utilizing the labeled data with the following theorem.

\begin{theorem}
	\label{theorem4}
	Fix $f\in\mathcal{F}$, for any $0<\delta<1$, with probability at least $1-\delta$, the generalization bound holds:
	\begin{equation}
	\scriptsize
	\begin{aligned}
	&\sum_{m\not\in \mathbf{C}_{\mathbf{P}_k}} \mathbb{E}_U^k\left[ P(f(\mathbf{x})\neq m)\right]- \frac{1}{n_U^k}\sum_{j=1}^{n^k} \sum_{m\not\in \mathbf{C}_{\mathbf{P}_k}} \mathbb{E}_j^k\left[ P'(f(\mathbf{x})\neq m)\right]\\
	\leq& \sum_{i\in\mathbf{C}_{\mathbf{P}_k}} \frac{\pi_i}{n^k_i} (1+\frac{\prod_{i\not\in \mathbf{C}_{\mathbf{P}_k}}\vert k-i\vert}{k^{C_{U_k}}}) \sum_{j=1}^{n^k_i}\sum_{m\not\in \mathbf{C}_{\mathbf{P}_k}} \mathbb{E}_j^k\left[ P'(f(\mathbf{x})\neq m)\right]\\
	&+(\sum_{i\in\mathbf{C}_{\mathbf{P}_k}} \pi_i+1) C V (\sum_{s\in \mathbf{C}_{\mathbf{P}_k}} \frac{1}{\sqrt{n_s^k}}\\&+ \frac{1}{\sqrt{n_U^k}})  + \sum_{i\in\mathbf{C}_{\mathbf{P}_k}}\pi_i (1+\frac{\prod_{i\not\in \mathbf{C}_{\mathbf{P}_k}}\vert k-i\vert}{k^{C_{U_k}}})\sqrt{\frac{log\frac{1}{\delta}}{2n_i^k}} + \sqrt{\frac{log\frac{1}{\delta}}{2n_U^k}}.
	\end{aligned}
	\end{equation}
\end{theorem}

Now we are ready to present the generalization bound for Eq.~\ref{equ:fedpurisk}.

\begin{theorem}
	\label{theorem5}
	As $n_i^k,n_U^k \rightarrow \infty$, $i \in\mathbf{C}_{\mathbf{P}_k} ,k\in \{1,...,K\}$, the generalization bound of the proposed FedPU is of order:
	\begin{equation}
	\mathcal{O}\left(\sum_{k=1}^KC^2(\sum_{i\in\mathbf{C}_{\mathbf{P}_k}}\frac{1}{\sqrt{n_i^k}}+ \frac{1}{\sqrt{n_U^k}}) \right).
	\end{equation}
\end{theorem}
It should be noted that for fully labeled data, the generalization bound using federated learning should be of order $\mathcal{O}\left(\sum_{k=1}^K(\frac{C^2}{\sqrt{\sum_{i\in\mathbf{C}_{\mathbf{P}_k}}n_i^k+ n_U^k}}) \right)$. As a result, the proposed method is no worse than $C\sqrt{C}$ times (assuming that each class has the same order of samples) of the fully-supervised models. Moreover, for the classical learning with fully labeled data (without federated learning), the generalization bound would be of order $\mathcal{O}\left(\frac{C^2}{\sqrt{\sum_{k=1}^K(\sum_{i\in\mathbf{C}_{\mathbf{P}_k}}n_i^k+ n_U^k})} \right)$. Therefore, the proposed method is no worse than $CK\sqrt{CK}$ times of the fully-supervised models without federated learning.

\section{Experiments}

In this section, we show the experimental results of the proposed method in both iid data and non-iid data on the MNIST and CIFAR-10 dataset. We also conduct ablation study to verify the effectiveness of the proposed method in different settings. 

We first detail the training strategy used in the following experiments. The SGD optimizer is used to train the network with momentum 0.5. For federated learning, we set the communication round as 200. For each client, the local epoch and local batchsize for training the network in each round is set as 1 and 100. The learning rate is initialized as 0.01 and exponentially decayed by 0.995 over communication rounds on the MNIST dataset. To show the effectiveness of the proposed method, we compare the proposed method with two different baselines. \textbf{Baseline-1} denotes that the network is trained using only positive data and FedAvg~\cite{Mcmahan2017communication}. \textbf{Baseline-2} denotes that the network is trained using fully-supervised data and FedAvg. \textbf{Baseline-3} denotes that the network is trained using conventional PU learning and FedAvg. Note that we also conduct experiments on FedSGD~\cite{Mcmahan2017communication} and FedProx~\cite{li2020federated}, which can be found in the supplementary materials.

\begin{table*}[t]
	\begin{center}
		\caption{Classification results with different number of positive classes in each client.}
		\label{table:dif}
		\begin{tabular}{|c|c|c|c|c|}
			\hline
			\textbf{Division of P-class} &\textbf{Overlap} &\textbf{Baseline-1}&\textbf{Proposed Method}&\textbf{Baseline-2}  \\
			\hline
			\hline
			[2,3,4,6,7,8] &\Checkmark& 93.84\% & 95.32\% & 97.91\% \\
			\hline
			[1,2,4,6,7] & \Checkmark&93.81\% & 95.01\% & 98.03\% \\
			\hline
			[2,4,6,8] &\Checkmark&  92.27\% & 95.28\% & 98.05\% \\
			\hline
			[3,7]&\XSolidBrush & 89.68\% & 94.68\% & 98.20\% \\
			\hline
			[2,3,5]&\XSolidBrush& 71.46\% & 93.65\% & 98.16\% \\
			\hline
			[1,2,3,4]&\XSolidBrush& 74.27\% & 94.48\% & 98.05\% \\
			\hline
		\end{tabular}
	\end{center}
\vspace{-0.5em}
\end{table*}

\subsection{Performance on iid Data with Balanced Positive Classes}

We evaluate our method in iid setting of federated learning, where the training data in each client is uniformly sampled from the original dataset. Specifically, we uniformly divide the training set into $K$ parts, where each part of data is class-imbalanced. Since the ability of each client is limited, only a few classes can be labeled. Moreover, only part of data in these classes is labeled. To fully investigate the ability of the proposed method, we conduct different settings as shown in Table~\ref{table:iid}, including using different number of clients ($\{2,4,5,10\}$) and different number of positive classes ($\{1,2,5,6,9\}$). We also investigate the influence of overlap of positive classes between different clients. Only half of data in each positive class is labeled.

\begin{table*}[t]
	\begin{center}
		\caption{Classification result on non-iid data.}
		\label{table:noniid}
		\begin{tabular}{|c|c|c|c|c|c|}
			\hline
			\textbf{Num of Partitions} &\textbf{Division of P-class}&\textbf{Overlap} &\textbf{Baseline-1}&\textbf{Proposed Method}&\textbf{Baseline-2}  \\
			\hline
			\hline
			5 &[2,2,...,2]&\Checkmark & 25.47\% & 91.67\% & 97.47\% \\
			\hline
			5 & [1,1,...,1]&\XSolidBrush &24.37\% & 91.24\% & 97.47\% \\
			\hline
			5& [4,4,3,3,2,2,1,1,1,1]&\Checkmark &76.92\% & 92.16\% & 97.47\% \\
			\hline
			2& [1,1,...,1] &\XSolidBrush & 69.24\% & 91.29\% & 96.19\% \\
			\hline
		\end{tabular}
	\end{center}

\end{table*}

We conduct experiments on the MNIST dataset, which is composed of images with $28\times28$ pixels from 10 categories. The MNIST dataset consists of 60,000 training images and 10,000 testing images. The results are shown in Table~\ref{table:iid}. We first investigate the setting that each client has overlap in positive classes and the number of clients varies from 2 to 10. 

The Baseline-1 trained with positive data can only achieve 85.47\%, 92.10\% and 93.15\% accuracies for 10, 4 and 2 clients, respectively. It can be seen that as the number of clients increases, the data is more discrete, which makes the accuracies of learned networks lower. Although the Baseline-2 can achieve higher performance (97.95\%, 98.05\% and 98.20\%), the networks should be trained with fully supervised data, which is usually unavailable in real-world applications. In contrast, the proposed method can achieve 92.50\%, 95.08\% and 95.37\% accuracies, respectively, which is consistently higher than those of the Baseline-1 and comparable to Baseline-2. 

We further investigate the non-overlap setting, where positive classes in each client are not overlaped. This setting is challenging, since the information of every class is contained in only one client. As a result, the Baseline-1 trained with positive data achieves only 37.13\%, 73.41\% and 74.00\% accuracies for 10, 5 and 2 clients, respectively. The proposed FedPU can still achieve 84.15\%, 93.45\% and 93.73\% accuracies by fully inheriting the information from the unlabeled data. These experiments show that the proposed method can perform well with iid data in federated setting. 

\begin{table*}[t]
	\begin{center}
		\small
		\caption{Classification results on CIFAR-10 dataset}
		\label{table:difcifar}
		\begin{tabular}{|c|c|c|c|c|c|c|}
			\hline
			\textbf{Data Distribution} &\textbf{Division of P-class} &\textbf{Overlap} &\textbf{Baseline-1}&\textbf{Proposed Method}&\textbf{Baseline-2} &\textbf{Baseline-3} \\
			\hline
			\hline
			iid & [2,2,2,2,2]& \XSolidBrush& 65.52\%& 76.81\%& 81.13\%&74.15\%\\
			\hline
			iid & [1,2,4,6,7]&\checkmark& 71.42\%&75.41 \%& 81.13\%&-\\
			\hline
			non-iid & [2,2,2,2,2,2,2,2,2,2] &\checkmark& 52.39\%& 61.05\%& 72.61\%&58.77\%\\
			\hline
			non-iid & [4,4,3,3,2,2,1,1,1,1]&\checkmark& 55.57\%& 65.73\%& 72.61\%&-\\
			\hline
		\end{tabular}
	\end{center}
\vspace{-0.5em}
\end{table*}

\begin{table*}[t]
	\begin{center}
		\caption{Comparison with semi-supervised methods on CIFAR-10 dataset}
		\label{table:comp}
		\begin{tabular}{|c|c|c|c|c|c|}
			\hline
			\textbf{Methods} &\textbf{Supervised FedAVG} &\textbf{UDA} &\textbf{FixMatch}&\textbf{FedMatch}&\textbf{Ours}  \\
			\hline
			\textbf{IID Acc.} & 80.25\%&  47.45\%  & 47.20\%&  52.13\%& 58.25\%\\
			\hline
			\textbf{Non-IID Acc.} & 84.70\%&   46.31\%  &  46.20\%&  52.25\%& 55.20\%\\
			\hline
		\end{tabular}
	\end{center}
\end{table*}

\subsection{Performance on iid Data with Imbalanced Positive Classes}
To further investigate the effectiveness of the proposed method, we study a more complicated setting that the number of positive classes is different in each client. The results are in shown in Table~\ref{table:dif}. For example, the division of P-class is $[2,3,4,6,7,8]$ means there are 6 clients consists of 2, 3, 4, 6, 7 and 8 positive classes, respectively. We also study both the overlap and non-overlap settings. 

The Baseline-1 achieves 93.84\%, 93.81\% and 92.27\% accuracies for different divisions of positive classes in the overlap setting, while the proposed method can achieve 95.32\%, 95.01\% and 92.28\% accuracies, respectively, which is much higher than Baseline-1. The results in non-overlap setting is worse than those in overlap setting, which is consistent with the results in Table~\ref{table:iid} where the number of positive classes is the same in each client. The Baseline-1 achieves only 89.68\%, 71.46\% and 74.27\% accuracies. When the number of clients grows, the performance of Baseline-1 drops dramatically. In contrast, the proposed method can achieve 94.68\%, 93.65\% and 94.48\% accuracies, which surpasses those of the Baseline-1 and is stable with different numbers of clients. Note that although the Baseline-2 can achieve a $\sim$98\% accuracy in all settings, it should trained with fully supervised data and violate most of the scenarios in real-world applications.

\subsection{Performance on Non-iid Data}

Another important setting for federated learning is that the data in different client is under the non-iid distribution. Therefore, we follow the settings in~\cite{li2018federated} to construct the non-iid data, where the data is sorted by class and divided to create two extreme cases: (a) 5-class non-iid, where the sorted data is divided into 50 partitions and each client is randomly assigned 5 partitions from 5 classes. (b) 2-class non-iid, where the sorted data is divided into 20 partitions and each client is randomly assigned 2 partitions from 2 classes. 

For 5-class non-iid, we study three different divide settings for positive classes in each client, which is shown in Table~\ref{table:noniid}. Compared with the iid setting, the non-iid setting is more challenging since the data distribution in each client is different and it is hard for the model to effectively learn the latent distribution on the whole dataset. Therefore, Baseline-1 can achieve only 25.47\%, 24.37\% and 79.62\% accuracies when dealing with non-iid and unlabeled data, which is hard to optimize. In contrast, the proposed method can still achieve 91.67\%, 91.24\% and 92.16\% accuracies, respectively, which outperforms Baseline-1 by a large margin. 

For 2-class non-iid, Baseline-1 achieves a 69.24\% accuracy while the proposed method achieves a 91.29\% accuracy, which still shows the superiority of the proposed FedPU. It should be noted that although the baseline method can achieve $\sim98\%$ accuracy, it requires the fully labeled data to train the model in each client. In contrast, the proposed method requires only a small amount of labeled data and utilizes the information on unlabeled data to learn an effective model. In conclusion, the proposed method successfully learns the latent distribution from the positive and unlabeled data in the non-iid federated setting and achieve better performance than conventional federated learning methods.

\subsection{Ablation Study}
\begin{table}[t]
	\begin{center}
		\caption{Classification results with different percentage of positive samples.}
		\label{table:per}
		\begin{tabular}{|c|c|c|c|}
			\hline
			\textbf{Percentage}  &\textbf{Baseline-1}&\textbf{FedPU}&\textbf{Baseline-2}  \\
			\hline
			\hline
			1/3 & 91.22\% & 94.46\% & 98.05\% \\
			\hline
			1/2 & 93.24\% & 95.31\% & 98.05\% \\
			\hline
			2/3 & 94.11\% & 95.60\% & 98.05\% \\
			\hline
		\end{tabular}
	\end{center}
\vspace{-1.0em}
\end{table}

In the above sections, we study the PU setting where there are half of data in each positive classes are labeled on the MNIST dataset. Here we make an ablation study to investigate the impact of the percentage of labeled data in positive class. We use 4 clients whose number of positive classes are all equal to 6. The data is collected with iid distributions from each client. As shown in Table~\ref{table:per}, with the growth of the percentage of labeled data (from 1/3 to 2/3), the accuracy of the proposed method can be improved from 94.46\% to 95.60\%, which indicates the effectiveness of the proposed method with different percentage of labeled data.

\subsection{Experiments on CIFAR-10}

After investigating the performance of the proposed FedPU on MNIST dataset, we further evaluate our method on the CIFAR-10 dataset. The CIFAR-10 dataset consists of 50,000 training images and 10,000 testing images with size $32\times32\times3$ from 10 categories. The training strategy is the same as that on the MNIST dataset. As shown in Table~\ref{table:difcifar}, experiments on different settings (\emph{e.g.}, data distribution, division of positive classes, overlap), are conducted to evaluate the effectiveness of the proposed method. 

We first investigate the results on the iid data. 5 clients with 2 positive classes in each client are used to train the model. The positive classes have no overlap. The FedAvg method trained with positive data achieves only a 62.52\% accuracy. The proposed method achieves a 76.71\% accuracy with the help of unlabeled data, which is more close to the result (81.13\%) trained with fully supervised data. Then, we turn to explore the challenging situation that each client has different number of positive class ([1,2,4,6,7]). The proposed method still achieves a 75.41\% accuracy, which is much higher than that of Baseline-1 (71.42\%).

We further construct the non-iid data on CIFAR-10 dataset following~\cite{li2018federated}, where each class of the training data is randomly divided into 5 partitions (50 partitions for 10 classes) and each client is randomly assigned 5 partitions from 5 classes. We also investigate the situation that each client has the same/different number of positive classes. As shown in Table~\ref{table:difcifar}, the models trained by the proposed method achieve accuracies of 61.05\% and 65.73\% and surpass those trained with the baseline method by a large margin (8.66\% and 10.16\%). In conclusion, the proposed method significantly improves the performance of the existing federated learning method in different settings on CIFAR-10 dataset.

\subsection{Comparison with Semi-supervised Methods}

To further show the superiority of the proposed method, we conduct comparison with the semi-supervised algorithms in federated setting. We follow the setting in~\cite{jeong2020federated} to use CIFAR-10 datasets. Specifically, 5 labeled images are extracted in per class for each client (100 clients) and the rest of images are used as unlabeled data. Table~\ref{table:comp} shows the performance of FedAVG using supervised data, UDA~(\cite{xie2019unsupervised}), FixMatch~(\cite{sohn2020fixmatch}), FedMatch~(\cite{jeong2020federated}) and the proposed FedPU. The proposed method achieve the state-of-the-art performance among all semi-supervised methods.

\begin{table}[h]
	\begin{center}
\vspace{-1.0em}
		\caption{Classification result on CIFAR-10 dataset.}
		\label{tab:class}
		\begin{tabular}{|c|c|c|}
			\hline
			\textbf{Data Distribution} &Iid & Non-iid \\
			\hline
			\hline
			\textbf{Baseline-1}&65.52\%&52.39\%\\
			\hline
			\textbf{FedPU $\pi$=0.1}&76.51\%&61.05\%\\
			\hline
			\textbf{FedPU $\pi$=0.05}&75.37\%& 60.13\%\\
			\hline
			\textbf{FedPU $\pi$=0.08}&76.43\%&60.67\%\\
			\hline
			\textbf{Baseline-2} &81.13\%&72.61\% \\
			\hline
			\textbf{Baseline-3}&74.15\%&58.77\%\\
			\hline	
		\end{tabular}
	\end{center}
\vspace{-1.0em}
\end{table}

\subsection{Results with Different Class Prior}

The class priors are necessary for applying the proposed method, which is assumed to be given. When the class priors are unknown, they can be estimated following~\cite{du2014class}. Therefore, we further analyze the sensitivity of the estimated class prior. Table~\ref{tab:class} shows the results of the proposed method using different class prior. The proposed method can achieve similar performance using different class priors and achieve the best performance when the class prior is known ($\pi$=0.1), which suggest the proposed method is robust with different estimated class priors. 

\section{Conclusion}

We study a real-world setting in federated learning problem, where each client could only label limited number of data in part of classes. Existing federated learning algorithms can hardly achieve satisfying performance since they cannot minimize the expected risk for each class in each client. To address this problem, we propose the Federated learning with Positive and Unlabeled data (FedPU) algorithm, which can effectively learn from both labeled and unlabeled data for each client. Theoretical analysis and empirical experiments demonstrate that the proposed method can achieve better performance than the conventional federated learning method learned by the positive data.  

\bibliography{example_paper}
\bibliographystyle{icml2022}

\appendix

\begin{table*}[h]
	\small
	\begin{center}
		\caption{Classification result on iid data.}
		\label{table:iidsup}
		\begin{tabular}{|c|c|c|c|c|c|c|}
			\hline
			\textbf{Num of Clients} &\textbf{Num of P-class}&\textbf{Overlap} &\textbf{Baseline-1}&\textbf{Proposed Method}&\textbf{Baseline-2}  \\
			\hline
			\hline
			10 &2&\Checkmark & 89.59\% & 90.25\% & 97.95\% \\
			\hline
			4 & 6&\Checkmark & 94.08\% & 94.22\% & 98.05\% \\
			\hline
			2 & 9&\Checkmark &94.36\% & 94.57\% & 98.20\% \\
			\hline
			10 & 1&\XSolidBrush&74.07\% & 89.78\% & 97.95\% \\
			\hline
			
			5 & 2&\XSolidBrush&90.58\% & 94.09\% & 98.03\% \\
			\hline
			
			2 & 5&\XSolidBrush&94.38\% & 94.62\% & 98.20\% \\
			\hline
		\end{tabular}
	\end{center}
	\vspace{-2.0em}
\end{table*}

\begin{table*}[h]
	\begin{center}
		\caption{Classification results with different number of positive classes in each client.}
		\label{table:difsup}
		\begin{tabular}{|c|c|c|c|c|}
			\hline
			\textbf{Division of P-class} &\textbf{Overlap} &\textbf{Baseline-1}&\textbf{Proposed Method}&\textbf{Baseline-2}  \\
			\hline
			\hline
			[2,3,4,6,7,8] &\Checkmark& 89.21\% & 93.33\% & 97.91\% \\
			\hline
			[1,2,4,6,7] & \Checkmark&90.38\% & 93.53\% & 98.03\% \\
			\hline
			[2,4,6,8] &\Checkmark&  93.20\% & 94.74\% & 98.05\% \\
			\hline
			[3,7]&\XSolidBrush & 93.56\% & 94.08\% & 98.20\% \\
			\hline
			[2,3,5]&\XSolidBrush& 89.24\% & 93.24\% & 98.16\% \\
			\hline
			[1,2,3,4]&\XSolidBrush& 89.18\% & 93.72\% & 98.05\% \\
			\hline
		\end{tabular}
	\end{center}
\end{table*}

\begin{table*}[h]
	\begin{center}
		\caption{Classification result on non-iid data.}
		\label{table:noniidsup}
		\begin{tabular}{|c|c|c|c|c|c|}
			\hline
			\textbf{Num of Partitions} &\textbf{Division of P-class}&\textbf{Overlap} &\textbf{Baseline-1}&\textbf{Proposed Method}&\textbf{Baseline-2}  \\
			\hline
			\hline
			5 &[2,2,...,2]&\Checkmark & 90.17\% & 92.54\% & 97.47\% \\
			\hline
			5 & [1,1,...,1]&\XSolidBrush &61.70\%& 89.69\% & 97.47\% \\
			\hline
			5& [4,4,3,3,2,2,1,1,1,1]&\Checkmark &81.74\% & 90.48\% & 97.47\% \\
			\hline
			2& [1,1,...,1] &\XSolidBrush & 85.07\% & 88.62\% & 96.19\% \\
			\hline
		\end{tabular}
	\end{center}
\end{table*}

\begin{table*}[h]
	\begin{center}
		\caption{Classification result using FedProx.}
		\label{table:FedProx}
		\begin{tabular}{|c|c|c|c|}
			\hline
			\textbf{Stragglers} &\textbf{Baseline-1}&\textbf{Proposed Method}&\textbf{Baseline-2}  \\
			\hline
			\hline
			0\% & 43.41\% & 47.50\% & 69.76\% \\
			\hline
			50\%  &41.72\%& 43.05\% & 67.81\% \\
			\hline
			90\%  &39.82\% & 41.35\% & 62.46\% \\
			\hline
		\end{tabular}
	\end{center}
\end{table*}

\section{Proofs}

\begin{theorem}
	\label{theorem1sup}
	Fix $f\in\mathcal{F}$, for any $0<\delta<1$, with probability at least $1-\delta$, the generalization bound holds:
	\begin{equation}
		\begin{aligned}
			&\mathbb{E}_i^k\left[P(f(\mathbf{x})\neq i) - \sum_{m\not\in \mathbf{C}_{\mathbf{P}_k}}P(f(\mathbf{x})\neq m)\right] \\&- \frac{1}{n_i^k}\sum_{j=1}^{n_i^k}\left[ P(f(\mathbf{x}_j)\neq i) - \sum_{m\not\in \mathbf{C}_{\mathbf{P}_k}}P(f(\mathbf{x}_j)\neq m) \right] \\
			\leq& 2CV(\sum_{s\in \mathbf{C}_{\mathbf{P}_k}} \frac{1}{\sqrt{n_s^k}} + \frac{1}{\sqrt{n_U^k}}) + \sqrt{\frac{log\frac{1}{\delta}}{2n_i^k}},
		\end{aligned}
	\end{equation}
	where $i\in \mathbf{C}_{\mathbf{P}_k}$, $V$ is a constant related to the VC-dimension of $f$ and the bound of loss function $l$, $n_i^k$ and $ n_U^k$ denotes the number of samples in $i$-class and unlabeled classes in $k$-th client, respectively.  
\end{theorem}
\begin{proof}
	According to~\cite{koltchinskii2002empirical}, denote $R(f)$ as the generalization error of hypothesis $f$, $\hat{R}_{S,\rho}(f)$ as its empirical margin loss with bound $\rho$, and $\mathcal{R}_m(f)$ as Rademacher complexity of the family of loss functions $f$, with probability at least $1-\delta$, we have:
	\begin{equation}
		R(f) \leq \hat{R}_{S,\rho}(f) + \frac{4C}{\rho} \mathcal{R}_n(f) + \sqrt{\frac{log\frac{1}{\delta}}{n}},
	\end{equation}
	where $n$ is the number of training samples and $C$ is the number of classes. 
	
	According to~\cite{bousquet2003introduction}, we have:
	\begin{equation}
		\mathcal{R}_m(f) \leq V' \sqrt{\frac{d}{n}},
	\end{equation}
	where $d$ is the Vapnik–Chervonenkis (VC) dimension of $f$, $V'$ is a constant. Taking $m$ in $n_i^k$ and $n_U^k$, we have:
	\begin{equation}
		\mathcal{R}_m(f) \leq V' \sqrt{d} (\sum_{s\in \mathbf{C}_{\mathbf{P}_k}} \frac{1}{\sqrt{n_s^k}} + \frac{1}{\sqrt{n_U^k}}). 
	\end{equation}
	Taking $V=4V'\frac{\sqrt{d}}{\rho}$, we then finish the proof.
\end{proof}

\begin{theorem}
	\label{theorem2sup}
	Fix $f\in\mathcal{F}$, for any $0<\delta<1$, with probability at least $1-\delta$, the generalization bound holds:
	\begin{equation}
		\begin{aligned}
			\mathbb{E}_i^k   \left[ P(f(\mathbf{x})\neq m)\right] - \frac{1}{n_i^k}\sum_{j=1}^{n_i^k}P(f(\mathbf{x}_j)\neq m) 
			\\\leq CV(\sum_{s\in \mathbf{C}_{\mathbf{P}_k}} \frac{1}{\sqrt{n_s^k}} + \frac{1}{\sqrt{n_U^k}}) + \sqrt{\frac{log\frac{1}{\delta}}{2n_i^k}}.
		\end{aligned}
	\end{equation}
\end{theorem}
The proof of Theorem 2 is the same as that of Theorem 1. 

\begin{lemma}
	\label{lemma3sup}
	Define 
	\begin{equation}
		P'(f(\mathbf{x})\neq m) = \frac{k^{C_{U_k}}}{k^{C_{U_k}}+\prod_{i\not\in \mathbf{C}_{\mathbf{P}_k}}\vert k-i\vert} P(f(\mathbf{x})\neq m),
	\end{equation}
	where $C_{U_k}$ denotes the number of unlabeled class in client $k$. The decomposition is hold:
	\begin{equation}
		\small
		\begin{aligned}
			&\sum_{m\not\in \mathbf{C}_{\mathbf{P}_k}} \mathbb{E}_U^k\left[ P(f(\mathbf{x})\neq m)\right] \\
			=&\sum_{i\in\mathbf{C}_{\mathbf{P}_k}} \pi_i (\frac{\prod_{i\not\in \mathbf{C}_{\mathbf{P}_k}}\vert k-i\vert}{k^{C_{U_k}}}) \sum_{m\not\in \mathbf{C}_{\mathbf{P}_k}} \mathbb{E}_i^k\left[ P'(f(\mathbf{x})\neq m)\right]\\&+ \sum_{m\not\in \mathbf{C}_{\mathbf{P}_k}} \mathbb{E}_U^k\left[ P'(f(\mathbf{x})\neq m)\right].
		\end{aligned}
	\end{equation}
\end{lemma}
\begin{proof}
	Given 
	\begin{equation}
		P'(f(\mathbf{x})\neq m) = \frac{k^{C_{U_k}}}{k^{C_{U_k}}+\prod_{i\not\in \mathbf{C}_{\mathbf{P}_k}}\vert k-i\vert} P(f(\mathbf{x})\neq m),
	\end{equation}
	we have:
	\begin{equation}
		\small
		\begin{aligned}
			& \mathbb{E}_U^k\left[ P(f(\mathbf{x})\neq m)\right] \\
			=&\int \sum_y  \frac{k^{C_{U_k}}+\prod_{i\not\in \mathbf{C}_{\mathbf{P}_k}}\vert k-i\vert}{k^{C_{U_k}}} P'(f(\mathbf{x})\neq m) p(\mathbf{x},y) d\mathbf{x}\\
			= &\int  P'(f(\mathbf{x})\neq m) \left[\sum_{j=1}^K\frac{k^{C_{U_k}}+\prod_{i\not\in \mathbf{C}_{\mathbf{P}_k}}\vert k-i\vert}{k^{C_{U_k}}} p(\mathbf{x},y=j)\right] d\mathbf{x}  \\
			= &\int  P'(f(\mathbf{x})\neq m) \sum_{j\in \mathbf{C}_{\mathbf{P}_k}}\frac{\prod_{i\not\in \mathbf{C}_{\mathbf{P}_k}}\vert k-i\vert}{k^{C_{U_k}}} p(\mathbf{x},y=j) d\mathbf{x}\\
			&+ \int  P'(f(\mathbf{x})\neq m) \sum_{j=1}^K p(\mathbf{x},y=j)d\mathbf{x}   \\
			=&\sum_{i\in\mathbf{C}_{\mathbf{P}_k}} \pi_i (\frac{\prod_{i\not\in \mathbf{C}_{\mathbf{P}_k}}\vert k-i\vert}{k^{C_{U_k}}})  \mathbb{E}_i^k\left[ P'(f(\mathbf{x})\neq m)\right]\\
			&+  \mathbb{E}_U^k\left[ P'(f(\mathbf{x})\neq m)\right].
		\end{aligned}
	\end{equation}
\end{proof}
\begin{theorem}
	\label{theorem4sup}
	Fix $f\in\mathcal{F}$, for any $0<\delta<1$, with probability at least $1-\delta$, the generalization bound holds:
	\begin{equation}
		\scriptsize
		\begin{aligned}
			&\sum_{m\not\in \mathbf{C}_{\mathbf{P}_k}} \mathbb{E}_U^k\left[ P(f(\mathbf{x})\neq m)\right]- \frac{1}{n_U^k}\sum_{j=1}^{n^k} \sum_{m\not\in \mathbf{C}_{\mathbf{P}_k}} \mathbb{E}_j^k\left[ P'(f(\mathbf{x})\neq m)\right]\\
			\leq& \sum_{i\in\mathbf{C}_{\mathbf{P}_k}} \frac{\pi_i}{n^k_i} (1+\frac{\prod_{i\not\in \mathbf{C}_{\mathbf{P}_k}}\vert k-i\vert}{k^{C_{U_k}}}) \sum_{j=1}^{n^k_i}\sum_{m\not\in \mathbf{C}_{\mathbf{P}_k}} \mathbb{E}_j^k\left[ P'(f(\mathbf{x})\neq m)\right]\\
			&+(\sum_{i\in\mathbf{C}_{\mathbf{P}_k}} \pi_i+1) C V (\sum_{s\in \mathbf{C}_{\mathbf{P}_k}} \frac{1}{\sqrt{n_s^k}}\\&+ \frac{1}{\sqrt{n_U^k}})  + \sum_{i\in\mathbf{C}_{\mathbf{P}_k}}\pi_i (1+\frac{\prod_{i\not\in \mathbf{C}_{\mathbf{P}_k}}\vert k-i\vert}{k^{C_{U_k}}})\sqrt{\frac{log\frac{1}{\delta}}{2n_i^k}} + \sqrt{\frac{log\frac{1}{\delta}}{2n_U^k}}.
		\end{aligned}
	\end{equation}
\end{theorem}
With the evidence of Lemma 3, the proof of Theorem 4 is the same as that of Theorem 1. 

\begin{theorem}
	\label{theorem5sup}
	As $n_i^k,n_U^k \rightarrow \infty$, $i \in\mathbf{C}_{\mathbf{P}_k} ,k\in \{1,...,K\}$, the generalization bound of the proposed FedPU is of order:
	\begin{equation}
		\mathcal{O}\left(\sum_{k=1}^KC^2(\sum_{i\in\mathbf{C}_{\mathbf{P}_k}}\frac{1}{\sqrt{n_i^k}}+ \frac{1}{\sqrt{n_U^k}}) \right).
	\end{equation}
\end{theorem}
By concluding the result in Theorem 1, 2 and 4. We can derive that the generalization bound in $k$-th client as $\mathcal{O}\left(C^2(\sum_{i\in\mathbf{C}_{\mathbf{P}_k}}\frac{1}{\sqrt{n_i^k}}+ \frac{1}{\sqrt{n_U^k}}) \right)$. By summing the bound in each client, we then finish the proof.

\section{Results on FedSGD}

We conduct the proposed method and baseline using FedSGD~\cite{Mcmahan2017communication}. The results are shown in Table~\ref{table:iidsup},~\ref{table:difsup} and~\ref{table:noniidsup}, which is consistent with those using FedAvg in the main paper.

We evaluate our method in iid setting of federated learning, where the training data in each client is uniformly sampled from the original dataset. The Baseline-1 trained with positive data can only achieve 89.59\%, 90.25\% and 90.25, 94.22\%, 94.57\% and 95.37\% accuracies, respectively, which is consistently higher than those of the Baseline-1 and comparable to Baseline-2. We further investigate the non-overlap setting. As a result, the Baseline-1 trained with positive data achieves only 74.07\%, 90.58\% and 94.38\% accuracies for 10, 5 and 2 clients, respectively. The proposed FedPU can still achieve 89.78\%, 94.09\% and 95.62\% accuracies by fully inheriting the information from the unlabeled data. These experiments show that the proposed method can perform well with iid data in federated setting. 

To further investigate the effectiveness of the proposed method, we study a more complicated setting that the number of positive classes is different in each client. The results are in shown in Table~\ref{table:difsup}. The Baseline-1 achieves lower performance than the proposed method. The results in non-overlap setting is worse than those in overlap setting, which is consistent with the results in Table~\ref{table:iidsup}. In contrast, the proposed method can surpasses those of the Baseline-1 and is stable with different numbers of clients.

Another important setting for federated learning is that the data in different client is under the non-iid distribution. Compared with the iid setting, the non-iid setting is more challenging since the data distribution in each client is different and it is hard for the model to effectively learn the latent distribution on the whole dataset. The proposed method can still outperforms Baseline-1 by a large margin, which is shown in Table~\ref{table:noniidsup}.

\section{Results on FedProx}

To further demonstrate the effectiveness of the proposed method, we conduct the proposed method and baseline using FedProx~\cite{li2020federated}. We test the non-iid settings using the division of $\left[2,2,...,2\right]$. We use 100 clients and each client has 5 partitions. The select rate and positive rate is set as 0.1. Table~\ref{table:FedProx} shows the experimental results. Under the different percentage of straggles, our methods stably outperform the baseline-1, which demonstrate the generality of the proposed method in different federated learning method,

\end{document}